\documentclass[conference]{IEEEtran}
\IEEEoverridecommandlockouts

\usepackage{amsmath,amssymb,amsthm}
\usepackage{graphicx}
\usepackage{booktabs}
\usepackage{colortbl}
\usepackage{xcolor}
\usepackage{enumitem}
\usepackage[hidelinks]{hyperref}
\usepackage{cite}

\definecolor{cBlue}{RGB}{70,130,180}
\definecolor{cGreen}{RGB}{50,140,60}
\definecolor{cRed}{RGB}{190,50,50}

\newtheorem{theorem}{Theorem}
\newtheorem{proposition}{Proposition}
\newtheorem{definition}{Definition}
\newtheorem{lemma}{Lemma}

\hypersetup{
  pdftitle={When More Evidence Hurts: Publication-Bias Drift and Principled Stopping for Biomedical Causal Search},
  pdfauthor={Fred Sun, Shangqi Guo}
}

\begin{document}

\title{When More Evidence Hurts: Publication-Bias Drift and Principled Stopping for Biomedical Causal Search}

\author{\IEEEauthorblockN{Fred Sun, Shangqi Guo\textsuperscript{*}}
\IEEEauthorblockA{Center for Brain-Inspired Computing Research\\
Department of Precision Instrument\\
Tsinghua University, Beijing, China}
\thanks{\textsuperscript{*}Corresponding author: Shangqi Guo.}}

\maketitle

\begin{abstract}
Automated biomedical evidence synthesis depends on retrieving published studies, but the biomedical literature is systematically skewed toward positive findings. Deeper retrieval can therefore make a system \emph{more} likely to falsely infer benefit when the true effect is null. We formalise this phenomenon as \emph{evidence drift} and prove that, under a standard publication-bias model, the false-positive probability on null-effect queries follows a strictly increasing large-sample envelope in retrieval depth, approaching one. Empirically, on a held-out test set of 140 Cochrane-derived queries, drift rises monotonically from 7.9\% to 15.7\% as the retrieval budget grows from 3 to 20 steps, and concentrates in the null-effect class. We present DACG-agent, a drift-aware causal-graph agent that incrementally builds a causal knowledge graph from PubMed abstracts and applies a two-layer stopping policy with complementary roles: a KL-divergence monitor that detects posterior convergence (the accuracy layer), and a Bradley--Terry process reward model (PRM) whose online decline detection halts retrieval once evidence quality peaks (the efficiency layer). Against full-budget retrieval, DACG-agent reduces evidence drift from 15.7\% to 6.4\% and improves null-effect accuracy by 21 percentage points (40.0\%$\to$61.4\%) while using 67\% fewer retrieval steps; overall accuracy rises from 61.4\% to 69.3\% (95\% CI 61--77). A simulation confirms the drift result transfers from the analysed vote-counting aggregator to the deployed noisy-OR one.
\end{abstract}


\section{Introduction}
\label{sec:intro}

Automated biomedical evidence synthesis promises to accelerate clinical decision-making, yet its reliability depends on the quality of the retrieved literature. Biomedical journals publish positive results at substantially higher rates than null or negative findings~\cite{begg1988publication,song2010dissemination}, and this publication bias propagates into every system that draws on the published record, from the parametric priors absorbed into LLM weights to the abstracts returned by PubMed searches~\cite{hasenboehler2007bias}. Each retrieval step preferentially surfaces papers reporting positive outcomes, progressively tilting accumulated evidence towards \textit{Beneficial} regardless of the true effect. We call this failure mode \emph{evidence drift} and prove that under a standard bias model the probability of misclassifying a true NoEffect query follows a strictly increasing large-sample envelope in retrieval depth, approaching one (Theorem~\ref{thm:drift}).

Existing approaches do not address this problem. LLM zero-shot methods inherit the same positive-result skew from pre-training data (91.4\% Beneficial accuracy but only 42.9\% NoEffect on the held-out test queries); standard RAG pipelines compound the skew by injecting biased documents, reducing NoEffect accuracy to 31.4\%; and recent knowledge-graph agents~\cite{jiang2024kg,su2024kgarevion,zhang2025medkgent} focus on multi-step reasoning without any mechanism to detect bias accumulation during iterative retrieval. The question is not \emph{whether} to retrieve, since clinical evidence synthesis demands traceable provenance~\cite{song2010dissemination}, but \emph{when to stop} before bias overwhelms the signal.

We present DACG-agent (\textbf{D}rift-\textbf{A}ware \textbf{C}ausal-\textbf{G}raph Agent), which incrementally builds a causal knowledge graph from PubMed abstracts and applies a two-layer stopping policy: a KL-divergence monitor that detects posterior convergence (Proposition~\ref{prop:kl-energy}), and a Bradley--Terry process reward model (PRM) whose online decline detection halts retrieval as soon as evidence quality peaks (Theorem~\ref{thm:prm-bayes}). Our contributions are:
\begin{enumerate}[leftmargin=*,nosep]
\item A formal proof that publication bias drives the misclassification probability of NoEffect queries to one as retrieval depth grows (Theorem~\ref{thm:drift}), with a simulation confirming the result transfers from the analysed vote-counting aggregator to the deployed noisy-OR one.
\item A two-layer stopping framework whose layers address distinct objectives: KL convergence monitoring (Proposition~\ref{prop:kl-energy}) governs \emph{accuracy} by halting once the posterior stabilises, while online PRM decline detection (Theorem~\ref{thm:prm-bayes}) governs \emph{efficiency}, reaching the same accuracy in 14\% fewer steps; a component ablation isolates each.
\item Validation on 140 Cochrane-derived test queries (macro-F1, bootstrap CIs): adaptive stopping raises null-effect accuracy by 21\,pp over full-budget retrieval (40.0\%$\to$61.4\%), cuts drift from 15.7\% to 6.4\%, and uses 67\% fewer steps.
\item A biomedical entity-grounding analysis over all 302 benchmark entities revealing the grounding--drift interaction: the best-characterised clinical entities, which accumulate the most literature, are precisely those most exposed to drift (Section~\ref{sec:grounding}).
\end{enumerate}

\section{Background and Problem Formulation}
\label{sec:background}

\subsection{Causal Search on Knowledge Graphs}

A biomedical causal query asks whether one entity has a causal effect on another: given a query $Q = (E_1, R^*, E_2, C)$, does entity~$E_1$ affect entity~$E_2$ through relation type~$R^*$, subject to optional constraints~$C$? The answer is a label $L \in \{\textit{Beneficial},\, \textit{NoEffect},\, \textit{Harmful}\}$. These queries are \emph{causal} rather than merely associational because the reference evidence derives from randomised controlled trials (RCTs), the design that licenses interventional conclusions, as synthesised in Cochrane reviews. The system maintains a dynamic causal graph $G_t = (V_t, E_t)$ that evolves with each retrieval step~$t$. Nodes of $G_t$ are resolved biomedical entities (interventions, outcomes, and intermediate concepts); edges are directed causal relations. Every edge stores a relation type, a polarity $p \in \{+1, -1, 0\}$, a list of supporting evidence (PubMed identifiers, PMIDs, with per-extraction confidence), and an aggregated confidence score~$w_{e,t} \in [0,1)$ produced by the noisy-OR update of Eq.~\ref{eq:noisy-or} (Lemma~\ref{lem:range}). A causal path $\pi = (e_1, \ldots, e_k)$ through~$G_t$ carries a route strength
\begin{equation}
\label{eq:route-strength}
\mathrm{Str}(\pi,t) = \bigg(\prod_{i=1}^{k} w_{e_i,t}\bigg) \cdot \exp(-\rho\, k),
\end{equation}
where $w_{e_i,t}\in[0,1)$ are edge beliefs and $\exp(-\rho k)$, with fixed per-hop discount $\rho=1$, tempers longer chains: the standard degree-of-belief-along-a-path score of probabilistic KG reasoning~\cite{lao2010relational}. Because each $w_{e,t}<1$ the product already decays with length, so $\exp(-\rho k)$ acts only as a per-hop compositional-uncertainty discount (Lemma~\ref{lem:range} and Section~\ref{sec:approach} formalise the resulting energy validity). For \emph{does zinc reduce common-cold duration?}, a one-hop path (zinc$\xrightarrow{-}$duration, $w=0.8$) has $\mathrm{Str}=0.8\,e^{-1}=0.29$ and a two-hop path (via immune response, $w=0.7,0.6$) has $0.7\cdot0.6\cdot e^{-2}=0.057$; both cast a strength-weighted \textit{Beneficial} vote. The label posterior $q_t(L\mid Q)$ marginalises over all discovered paths, so retrieval becomes incremental posterior update over a dynamic graph, and monitoring how that posterior evolves makes the stopping problem well-defined.

\subsection{Evidence Drift Under Publication Bias}

The empirical drift pattern described in Section~\ref{sec:intro} admits a precise characterisation. We first formalise the notion of evidence drift.

\begin{definition}[Evidence Drift]
\label{def:drift}
Let $q_t(L \mid Q)$ be the label posterior after $t$ retrieval steps and $\hat{L}_t = \arg\max_L q_t(L \mid Q)$ the predicted label. For a query with true label~$L^\star$, the \emph{drift risk} at depth~$t$ is $D(t) = P(\hat{L}_t \neq L^\star \mid L^\star)$. The system exhibits \emph{evidence drift} for label~$L^\star$ if $D(t) \to 1$ as $t \to \infty$. In finite trajectories, a \emph{drift event} occurs when $\exists\, s < t$ such that $\hat{L}_s = L^\star$ but $\hat{L}_t \neq L^\star$: the system once held the correct answer and subsequently lost it.
\end{definition}

\noindent Suppose each retrieved paper independently reports a positive result with probability $\tfrac{1}{2}+b$ for some $b > 0$, regardless of the true causal effect. This models publication bias: positive findings are over-represented by a margin~$b$. Let $\hat{p}_t^+$ denote the fraction of positive papers after~$t$ retrieval steps.

\begin{theorem}[Drift Under Publication Bias]
\label{thm:drift}
Under the bias model above, a vote-counting aggregator predicts \textit{Beneficial} when $\hat{p}_t^+ > \tfrac{1}{2}$ and \textit{NoEffect} otherwise. For queries whose true label is \textit{NoEffect}, the drift risk (Definition~\ref{def:drift}) satisfies
\[
D(t) = P\!\big(\hat{p}_t^+ > \tfrac{1}{2}\big) \;\xrightarrow{\;\text{CLT}\;}\; \Phi\!\Big(\frac{b\sqrt{t}}{\sigma}\Big),
\]
where $\sigma^2 = (\tfrac{1}{2}+b)(\tfrac{1}{2}-b)$ and $\Phi$ is the standard normal cumulative distribution function (CDF), and the limit follows from the central limit theorem (CLT). The CLT approximation $\Phi(b\sqrt{t}/\sigma)$ is strictly increasing in~$t$, and $D(t) \to 1$ as $t \to \infty$.
\end{theorem}

\noindent\emph{Proof sketch.} The bias shifts the expected positive fraction above $\tfrac{1}{2}$ while variance shrinks as $O(1/t)$, so the aggregator becomes increasingly confident in the wrong label. The exact finite-sample $D(t)$ may exhibit minor non-monotonicity due to the discrete nature of binomial counts (the threshold $\lfloor t/2 \rfloor$ shifts with parity), but the CLT envelope is strictly increasing and becomes tight for moderate~$t$. Full proof in the extended version.

Under publication bias, additional data does not correct the estimate but rather increases confidence in the wrong answer; the convergence rate depends on the bias magnitude~$b$ and the aggregation mechanism.

\subsection{From Vote-Counting to the Deployed Aggregator}
\label{sec:noisyor}
Theorem~\ref{thm:drift} analyses a vote-counting aggregator, whereas DACG deploys a noisy-OR confidence update (Eq.~\ref{eq:noisy-or}). A natural concern is whether the drift result transfers. We simulate a true-null query as a stream of reports, each positive with probability $\tfrac12+b$, and compare the two aggregators over depth $T=20$ (8{,}000 trials each). The two produce statistically indistinguishable, rising false-positive curves (Fig.~\ref{fig:noisyor}): for $b=0.1$ both rise from 61\% to 81\% by $T=20$; for $b=0.2$, from 71\% to 97\%. The drift result therefore transfers to the deployed aggregator; we do \emph{not} claim noisy-OR is worse, only that it is not immune.

We further test three departures from the theorem's idealised assumptions (addressing the concern that i.i.d.\ retrieval and a fixed bias $b$ are unrealistic). (a) \emph{Correlated evidence} ($\rho=0.5$ between successive reports) slows drift (59\%$\to$70\%) but preserves monotonicity. (b) \emph{Query-varying bias} ($b\sim\mathcal{N}(0.1,0.05)$) reproduces the i.i.d.\ curve (60\%$\to$79\%). (c) \emph{Asymmetric extraction}, where null findings are extracted at lower confidence than positive ones ($s^-=0.3$ vs.\ $s^+=0.6$), the empirically realistic case since null results are both under-published and under-reported within abstracts, drives drift \emph{higher}, to nearly 100\% by $T=20$, though not strictly monotonically (a structural dip near $t=4$ from the tie threshold). The theorem is thus a conservative lower bound on the drift a deployed system faces.

\begin{figure}[t]
\centering
\includegraphics[width=\columnwidth]{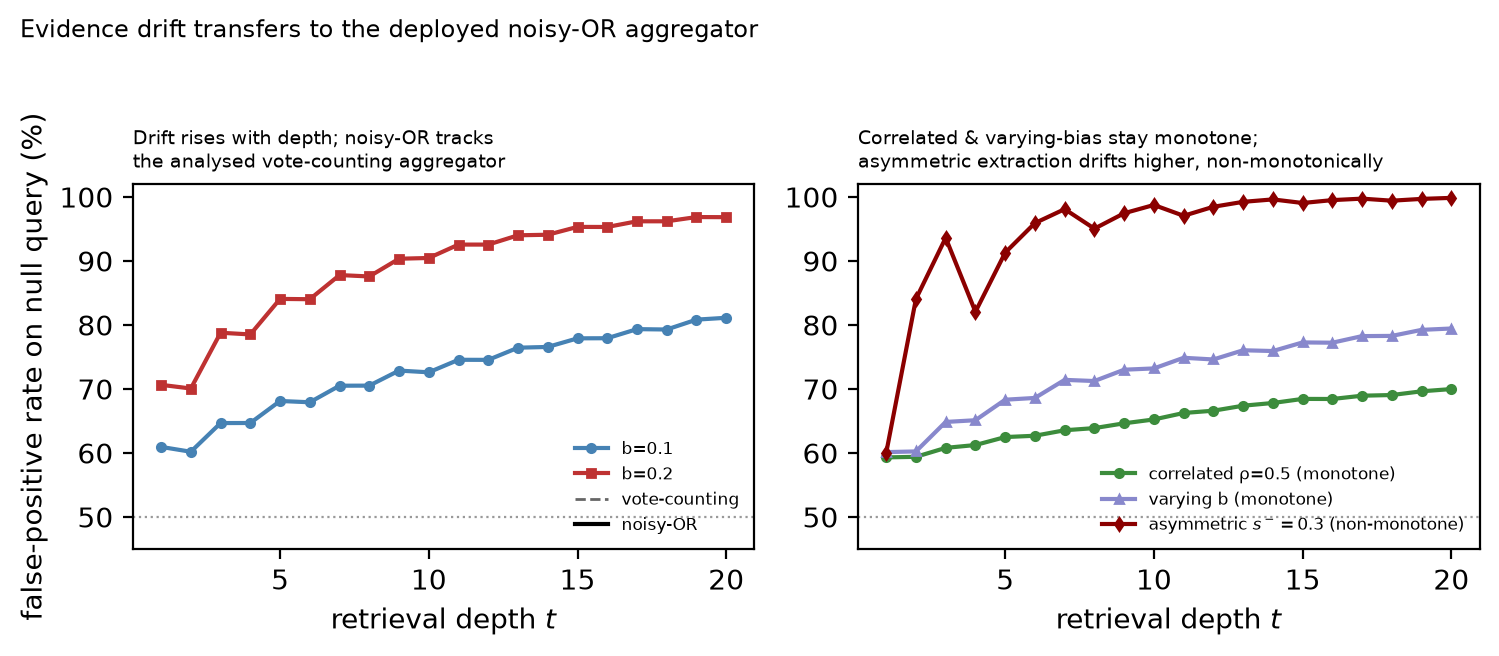}
\caption{Drift under different aggregators and assumptions. Left: vote-counting and noisy-OR yield indistinguishable rising drift, confirming the result transfers to the deployed aggregator. Right: correlated evidence and query-varying bias remain monotone; asymmetric extraction (realistic for under-reported nulls) drifts higher but non-monotonically, so the theorem is a conservative bound.}
\label{fig:noisyor}
\end{figure}

\section{Approach: Knowledge-Graph Search with Principled Stopping}
\label{sec:approach}

DACG-agent operates in an iterative loop: retrieve a batch of PubMed abstracts, extract causal triples, update the graph, infer paths, compute the posterior, and decide whether to stop or continue (Fig.~\ref{alg:kgsa}).

\subsection{Graph Construction}

An entity resolver normalises mentions through case folding, abbreviation expansion, and alias matching, mapping co-references to stable node identifiers (edge schema as in Section~\ref{sec:background}).

When new evidence with confidence~$s_{\mathrm{in}}$ supports an existing edge, the aggregated confidence updates through the noisy-OR rule:
\begin{equation}
\label{eq:noisy-or}
s_{\mathrm{new}} = 1 - (1 - s_{\mathrm{old}})(1 - s_{\mathrm{in}}).
\end{equation}
Each piece of evidence is treated as an independent chance of detecting the true relation.

\begin{lemma}[Range of edge belief]
\label{lem:range}
If every per-extraction confidence $s_j\in[0,1)$, then the aggregated edge belief $w_{e,t}=1-\prod_j(1-s_j)\in[0,1)$.
\end{lemma}
\noindent This resolves the soundness question for Eq.~\ref{eq:route-strength}: because $w_{e,t}<1$, the path-strength product already decays geometrically with length, and $\mathrm{Str}(\pi,t)\in(0,1]$, so the route energy $E_\pi=-\log\mathrm{Str}(\pi,t)\ge0$ (Eq.~\ref{eq:energy}) is a valid non-negative energy function and Eq.~\ref{eq:boltzmann} defines a proper Boltzmann distribution. The per-hop discount $\exp(-\rho)$, $\rho=1$, is therefore not required for probabilistic validity; it is a modelling refinement that re-weights paths of different lengths, and it affects at most 1.2\% of decision steps (extended version). The path-strength formulation follows the degree-of-belief-along-a-path score standard in probabilistic knowledge-graph reasoning~\cite{lao2010relational,gardner2014incorporating}. Conflicting polarity evidence between the same entity pair is retained as separate edges rather than cancelled, preserving the graph's ability to represent genuine disagreement.

\subsection{Path Inference and Label Posterior}

The system enumerates direct edges and two-hop paths $(E_1 \to v \to E_2)$ through intermediate nodes. Polarity propagates multiplicatively along each path, following composition rules (beneficial followed by harmful yields harmful, and so on). Route energies $E_\pi = -\log\mathrm{Str}(\pi,t)$ induce a Boltzmann distribution over paths:
\begin{equation}
\label{eq:boltzmann}
q_t(\pi \mid Q) \propto \exp\!\big({-\beta\, E_\pi}\big) = \mathrm{Str}(\pi,t)^\beta.
\end{equation}
Here $\beta>0$ is the Boltzmann inverse temperature that sharpens the distribution towards higher-strength paths; the deployed value is $\beta=1.5$. Each path~$\pi$ maps to a label through its composed polarity, giving a conditional $q_t(L \mid \pi, Q)$. The label posterior marginalises over the path distribution:
\begin{equation}
\label{eq:posterior}
q_t(L \mid Q) = \sum_{\pi \in \mathcal{P}_t} q_t(L \mid \pi, Q)\, q_t(\pi \mid Q),
\end{equation}
where $\mathcal{P}_t$ contains the top-5 paths ranked by route strength. The system distinguishes \textit{NoEvidence} (insufficient data; continue searching) from \textit{NoEffect} (sufficient evidence supports neutral polarity; a genuine finding).

\subsection{Energy Functional}

To reason about when retrieval has yielded enough information, we define an energy functional on the explanation subgraph $S_t \subseteq G_t$:
\begin{equation}
\label{eq:energy}
E(S_t) = \underbrace{-\!\log q_t(L \mid Q)}_{\text{log-posterior}} + \underbrace{\lambda_U \!\sum_{e \in S_t}\! H(\theta_e)}_{\text{edge uncertainty}} + \underbrace{\lambda_V\, \mathrm{Viol}(S_t; C)}_{\text{constraint violation}},
\end{equation}
where $H(\theta_e)$ is the polarity entropy of edge~$e$ and $\mathrm{Viol}(S_t; C) = \sum_{e \in S_t}\max(0,\, \tau_q - w_{e,t})$ penalises edges whose confidence falls below a query-specific threshold~$\tau_q$. The log-posterior term favours decisive conclusions, the uncertainty term penalises ambiguous edge polarities, and the constraint-violation term downweights explanations resting on poorly supported evidence. When the energy stabilises, further retrieval is unlikely to improve the explanation and may instead introduce bias.

\subsection{Stopping Criterion~I: KL Convergence}

The agent monitors the KL divergence between successive label posteriors:
\begin{equation}
\label{eq:kl}
\Delta_t = D_{\mathrm{KL}}\!\big(q_t(L) \,\|\, q_{t-1}(L)\big).
\end{equation}
KL divergence is a natural convergence signal: it is cheap to compute and directly measures whether retrieval is still changing the belief state. Small $\Delta_t$ indicates that the latest retrieval step did not materially alter the posterior. We deliberately monitor the \emph{rate of change} rather than the posterior magnitude: the posterior is poorly calibrated (expected calibration error $0.236$ on the test trajectories, overconfident at high-confidence bins), so its magnitude is an unreliable stopping signal whereas its stabilisation is not.

\begin{proposition}[KL--Energy Bound]
\label{prop:kl-energy}
If $q_t(L) \ge \varepsilon_0 > 0$ for all labels and steps, and the edge-uncertainty and constraint-violation terms change by at most~$\varepsilon_U$ and~$\varepsilon_V$ between consecutive steps, then
\begin{equation}
\label{eq:kl-bound}
\big|E(S_t) - E(S_{t-1})\big| \le \frac{1}{\varepsilon_0}\sqrt{\tfrac{1}{2}\, D_{\mathrm{KL}}(q_t \| q_{t-1})} + \lambda_U \varepsilon_U + \lambda_V \varepsilon_V.
\end{equation}
\end{proposition}

\noindent\emph{Proof sketch.}
Pinsker's inequality bounds total variation by the square root of KL divergence. The log-posterior is $1/\varepsilon_0$-Lipschitz on $[\varepsilon_0, 1]$. Combining with the assumed bounds on the remaining terms yields the result. Full proof in the extended version.\qed

\smallskip
Small KL divergence tells the agent that its beliefs have stabilised, but it cannot distinguish convergence to a correct answer from convergence to an incorrect one. A system that has drifted under publication bias may exhibit low KL divergence precisely because biased evidence has overwhelmed the signal. This limitation calls for a second, discriminative stopping criterion.

\subsection{Stopping Criterion~II: Process Reward Model}

Since a system can stabilise at the wrong answer under biased evidence accumulation, a process reward model (PRM) provides a complementary discriminative signal: it learns which graph states resemble historically correct stopping points. The PRM is a multi-layer perceptron mapping a 20-dimensional feature vector~$\phi_t$ to a scalar reward~$r(\phi_t)$. The feature vector captures path structure~(6), conflict indicators~(2), coverage statistics~(4), saturation measures~(2), global graph statistics~(4), and energy terms~(2); definitions appear in Appendix~\ref{app:impl}.

The PRM is trained on preference pairs derived from recorded trajectories. A snapshot at step~$i$ is preferred over a snapshot at step~$j$ (from the same query) when the step-$i$ conclusion is correct and the step-$j$ conclusion is not. Training uses the Bradley--Terry cross-entropy loss.

\begin{theorem}[PRM Bayes-Optimality]
\label{thm:prm-bayes}
Under a Bradley--Terry noise model with latent utility $U_t = \log\!\big[p(\text{correct} \mid \phi_t) / p(\text{incorrect} \mid \phi_t)\big]$, the population-optimal PRM satisfies $r^*(\phi_t) = U_t + c$ for a constant~$c$. The $\arg\max$ of~$r^*$ over the trajectory therefore coincides with the $\arg\max$ of the correctness probability.
\end{theorem}

\noindent\emph{Proof sketch.}
The Bradley--Terry loss is minimised when $r(\phi_t)$ equals the log-odds of being preferred plus a constant. Since preferences are generated by correctness, the optimal reward is an affine transform of the log-odds of correctness. The $\arg\max$ is invariant to affine shifts. Full proof in the extended version.\qed

\smallskip
In practice, the PRM is applied as an online decline detector rather than a retrospective argmax selector.  At each step the agent records $r(\phi_t)$; when the reward drops below the running maximum by more than a threshold~$\alpha$, the agent stops, because the graph state is moving away from configurations historically associated with correct conclusions.

\subsection{Combined Stopping Policy}

The two layers address complementary failure modes.  Layer~1 (KL convergence) triggers when $\Delta_t < \delta$, indicating that the posterior has stabilised.  Layer~2 (PRM decline) triggers when $r(\phi_t) < \max_{s \le t} r(\phi_s) - \alpha$, discriminating between correct and biased convergence.  A secondary PRM convergence rule fires when the last four rewards span less than a convergence threshold~$\alpha_c$, detecting plateaus that KL may miss.  The combined policy aims to halt retrieval before publication bias overwhelms the early signal. Figure~\ref{alg:kgsa} gives the complete procedure.

\begin{figure}[t]
\centering
\fbox{\parbox{0.92\columnwidth}{\small
\textbf{DACG-agent stopping loop}\\[2pt]
\textbf{Input:} query $Q$, budget $B$, thresholds $\delta,\alpha,\alpha_c$\\
1.\ Resolve entities; init $G_0{\leftarrow}\emptyset$, $q_0(L){\leftarrow}$Uniform, $r_{\max}{\leftarrow}{-}\infty$\\
2.\ \textbf{for} $t=1$ \textbf{to} $B$:\\
3.\quad Retrieve abstracts; extract triples; update $G_t$ (Eq.~\ref{eq:noisy-or})\\
4.\quad Compute path strengths (Eq.~\ref{eq:route-strength}), posterior $q_t(L{\mid}Q)$ (Eq.~\ref{eq:posterior})\\
5.\quad Compute $\Delta_t$ (Eq.~\ref{eq:kl}); extract $\phi_t$; $r_{\max}{\leftarrow}\max(r_{\max},r(\phi_t))$\\
6.\quad \textbf{if} $\Delta_t<\delta$: \textbf{break} \hfill(Layer 1: KL converged)\\
7.\quad \textbf{if} $r(\phi_t)<r_{\max}-\alpha$: \textbf{break} \hfill(Layer 2: PRM declining)\\
8.\quad \textbf{if} $t{\ge}4$ and 4-step reward range ${<}\alpha_c$: \textbf{break} \hfill(PRM converged)\\
9.\quad \textbf{if} $q_t$ yields \textit{NoEvidence}: trigger multi-hop expansion\\
10.\ \textbf{return} $\hat{L}=\arg\max_L q_t(L{\mid}Q)$ with provenance from $G_t$
}}
\caption{DACG-agent Retrieval Loop with Two-Layer Stopping}
\label{alg:kgsa}
\end{figure}

\section{Experimental Design}
\label{sec:setup}

\subsection{Benchmark}

The evaluation benchmark derives from Cochrane systematic reviews, each providing a consensus label (Beneficial, NoEffect, or Harmful) for an intervention-outcome pair. The benchmark comprises 969 review-derived queries, partitioned by review into 656 train / 140 validation / 173 test with no intervention--outcome pair shared across splits. For evaluation we apply an \emph{unambiguity filter} that retains queries where the gold label (structured review metadata) agrees with the ground-truth label (relabelled conclusion text), removing reviews with mixed or subgroup-dependent conclusions; on the test split this leaves 172 evaluable queries (140 Beneficial/NoEffect and 32 Harmful). \textbf{All accuracy results in this paper are reported on the held-out test split alone}, restricted to the Beneficial/NoEffect classes ($n=140$: 70 Beneficial, 70 NoEffect), eliminating train/validation leakage into the reported metrics. As a generalisation check we additionally report the disjoint validation+test union ($n=280$; extended version). The 32 Harmful queries are analysed separately because abstract-level synthesis reaches only 28\% oracle accuracy on them; the entity-grounding audit (Section~\ref{sec:grounding}) spans the full 172-query set. The retrieved evidence spans PubMed abstracts published 2010--2025.

Cochrane consensus labels are themselves meta-analysis products and can inherit the publication bias we study, making them a strong but imperfect gold standard. This does not undermine the drift analysis: a truly null effect mislabelled Beneficial by a biased review would only \emph{reduce} the measured drift on that query, so our reported drift is a lower bound against an unbiased ground truth.

To guarantee that every method is compared on identical evidence, the agent first runs the full 20-step trajectory without stopping, storing the complete graph state, posterior, features, and PRM reward at every step; each stopping rule is then applied offline to the same recorded sequence. PRM preference pairs are constructed exclusively from training-set trajectories.

\subsection{Biomedical Entity Grounding}
\label{sec:grounding}
Because DACG reasons over a knowledge graph whose nodes are intervention and outcome entities, entity-resolution fidelity is a clinical correctness concern, not a generic engineering detail. We therefore audit how well the extracted entities ground to the Medical Subject Headings (MeSH) controlled vocabulary. This graph-construction audit spans the full test benchmark---all 302 distinct head/tail entities across the complete 172-query test split, including the 32 Harmful queries that the accuracy evaluation analyses separately---since resolution fidelity is a property of the graph, not of the label subset. Querying NCBI for these 302 entities, 22.5\% (95\% CI 17.8--27.2) match a MeSH descriptor exactly and 43.0\% (37.5--48.6) after light lexical normalisation (case folding, qualifier and parenthetical stripping); the 57.0\% residual are predominantly non-atomic clinical phrases carried from review titles (e.g.\ ``antiplatelet agents and anticoagulants''), which no single descriptor covers (Fig.~\ref{fig:grounding}, left). Grounding also exposes synonymy the surface resolver misses: 11 surface forms collapse onto 5 shared MeSH concepts (e.g.\ ``migraine''/``acute migraine headaches''; and ``omega-3 fatty acids'', ``omega 3 fatty acid'', and ``polyunsaturated fatty acids'' onto one descriptor).

Grounding is not merely a data-cleaning step: it interacts with drift in a way that reinforces our central thesis. Splitting the same 172-query test split by whether \emph{both} endpoints ground to MeSH (33 both-grounded vs.\ 139 not), the well-grounded queries drift \emph{more}, not less (36.4\% vs.\ 17.3\%; Fig.~\ref{fig:grounding}, right). The reason is mechanistic rather than paradoxical: entities that map cleanly to a MeSH concept are the well-studied ones, and they accumulate roughly twice the evidence (12.1 vs.\ 6.3 abstracts per query). More evidence is exactly what Theorem~\ref{thm:drift} predicts drives a bias-agnostic aggregator toward Beneficial, so the best-characterised clinical entities, where a decision-support tool is most likely to be trusted, are precisely where drift-aware stopping matters most. Ontology grounding thus sharpens both the interpretability of the graph and the case for the stopping policy; a full MeSH/UMLS entity linker is the natural next step for clinical use (Section~\ref{sec:discussion}).

\begin{figure}[t]
\centering
\includegraphics[width=\columnwidth]{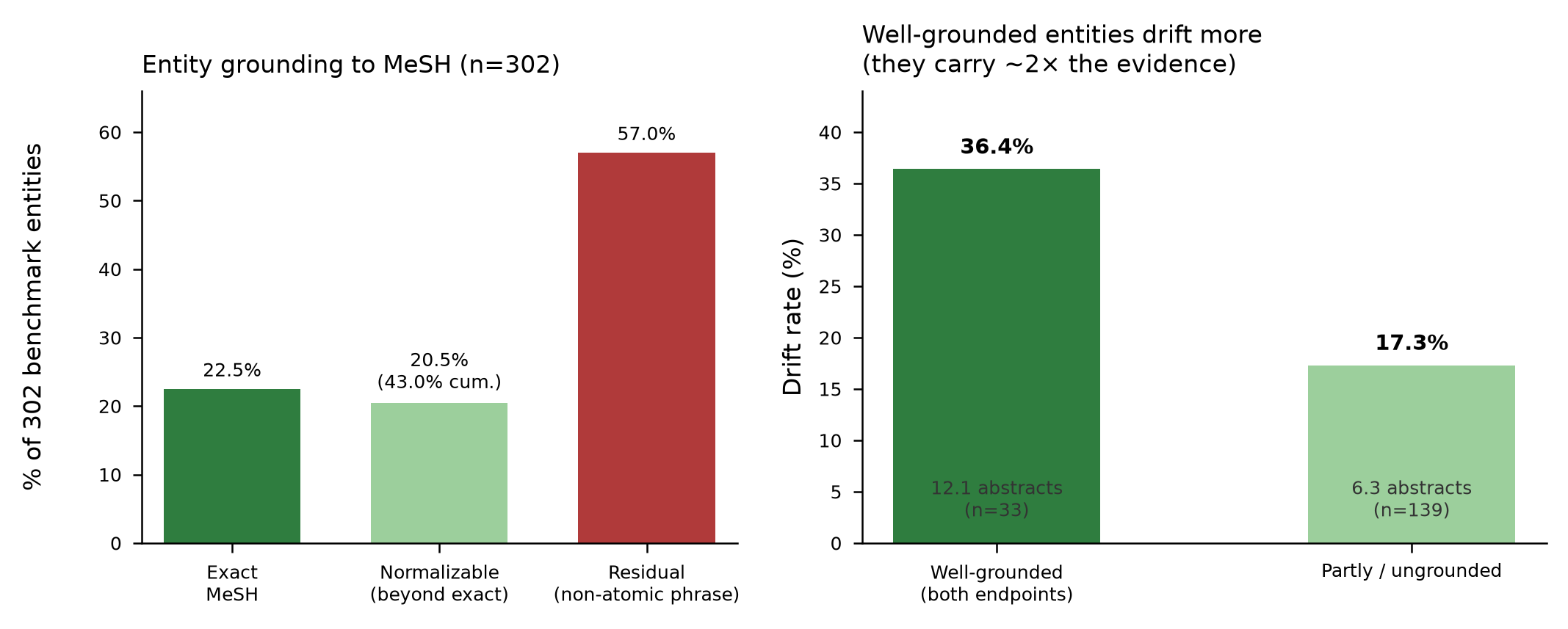}
\caption{Biomedical entity grounding on the full test benchmark ($n=302$ entities over all 172 test queries, Harmful included; distinct from the $n=140$ Beneficial/NoEffect accuracy evaluation). Left: 22.5\% of entities match a MeSH descriptor exactly and a further 20.5\% after light normalisation (43.0\% cumulative); the remaining 57.0\% are non-atomic clinical phrases. Right: queries whose head and tail both ground to MeSH drift \emph{more} (36.4\% vs.\ 17.3\%), because well-grounded (well-studied) entities carry roughly twice the retrieved evidence (12.1 vs.\ 6.3 abstracts), the more-evidence-more-drift mechanism of Theorem~\ref{thm:drift}.}
\label{fig:grounding}
\end{figure}

\subsection{Baselines}

Four \emph{external baselines} operate outside the adaptive-stopping framework. \textbf{E1}~(LLM Zero-Shot): Claude Sonnet answers the causal query with no retrieval. \textbf{E2}~(LLM+RAG): PubMed returns 10 abstracts; the LLM synthesises a judgment. \textbf{E3}~(Single-Pass KG-agent): a controlled instantiation of the paradigm shared by current LLM-driven medical KG agents, which commit to a conclusion from a single evidence-gathering pass or a pre-built graph, with no query-time, evidence-sufficiency stopping criterion~\cite{zhang2025medkgent,su2024kgarevion,jiang2024kg}. These systems differ substantially in their internals---KGARevion verifies LLM-generated triples against a static curated KG for multiple-choice QA, MedKGent constructs a temporally evolving KG offline---and none is openly reproducible end-to-end on the causal-direction task we study. We therefore implement the operative core they share (retrieve, extract causal triples, conclude, without adaptive stopping) on a backbone identical to DACG's, isolating the stopping policy from backbone strength; re-targeting a QA or construction system to our task would itself be a reimplementation, and conflate the two. Section~\ref{sec:drift-generalize} corroborates this on a different LLM (glm-4-flash). \textbf{E4}~(KGARevion): to pair the controlled paradigm with a genuine published system, we additionally run KGARevion~\cite{su2024kgarevion} end-to-end and unmodified on the 140 queries reframed as two-option questions, using its publicly released model and knowledge-graph weights as distributed. It is a Generate--Review--Revise--Answer agent that verifies LLM-extracted triples against a static curated biomedical KG (PrimeKG) and is closed-book with respect to the primary literature.

Four \emph{fixed-budget} baselines halt at $k \in \{3, 5, 10, 20\}$.

Three \emph{graph-structural} baselines monitor the evolving knowledge graph rather than the label posterior:
\textbf{Density stability} halts when the graph density ceases to change between consecutive steps;
\textbf{Path discovery} halts when the path-discovery rate drops (no new causal paths found);
\textbf{Evidence plateau} halts when the total evidence growth rate falls below a threshold over a sliding window.
These baselines test whether structural convergence of the graph suffices for accurate stopping.

Two \emph{ablation} variants isolate stopping components (Table~\ref{tab:ablation}): DACG$_{\text{KL}}$ (KL convergence only, no PRM) and DACG$_{\text{PRM}}$ (PRM decline only, no KL). An \emph{oracle} selects the per-query best step.

\subsection{Evaluation Metrics}

Because the test set is class-balanced (70/70) we report \emph{macro-F1} alongside accuracy, with 95\% bootstrap confidence intervals (10{,}000 resamples). Four further metrics capture retrieval behaviour. \textbf{Accuracy}: fraction matching Cochrane consensus. \textbf{Steps}: mean retrieval iterations. \textbf{Oracle regret}: $\max(0,\, \text{stop\_step} - \text{first\_correct\_step})$. \textbf{Drift rate}: fraction of queries that reach a correct answer at some step but are incorrect at the stopped step. Because the central claim concerns null-effect identification under bias, \emph{NoEffect accuracy} and \emph{drift rate} are the primary evaluation criteria. Overall accuracy is reported but is expected to favour methods exploiting the Beneficial-skewed prior.

\subsection{Reproducibility}
Implementation details appear in Appendix~\ref{app:impl}.

\section{Empirical Analysis}
\label{sec:results}

\subsection{Evidence Drift in Practice}
\label{sec:drift-empirical}

The asymmetry predicted by Theorem~\ref{thm:drift} is visible in the per-class mean posterior $q_t(\textit{Beneficial})$: for NoEffect queries it climbs steadily from 0.21 to 0.43 by step~14 (a $+0.22$ drift toward the boundary), whereas Beneficial queries start at the boundary (0.48) and rise only to 0.60. Drift thus concentrates in the NoEffect class: of NoEffect queries that reach a correct intermediate answer, 46.9\% drift to an incorrect one by step~20, versus only 10.0\% of Beneficial queries. Since the median time-to-correct is one step, the agent typically finds the right answer after a single retrieval and then loses it as biased evidence accumulates.

\subsection{Main Results}

Tables~\ref{tab:main} and~\ref{tab:ablation} present results with 95\% bootstrap confidence intervals (10{,}000 resamples).

\begin{table}[t]
\centering
\caption{Comparison on the held-out test set ($n=140$; 70 Beneficial, 70 NoEffect). Acc = overall accuracy with 95\% bootstrap CI; mF1 = macro-F1; Steps = mean retrieval steps; Drift = fraction reaching a correct answer then drifting away. Ben/NoE = per-class accuracy.}
\label{tab:main}
\footnotesize
\setlength{\tabcolsep}{0.6pt}
\begin{tabular}{@{}l cccccc@{}}
\toprule
Method & Acc\,(\%) & mF1 & Drift\,(\%) & Steps & Ben\,(\%) & NoE\,(\%) \\
\midrule
\multicolumn{7}{@{}l}{\emph{External baselines (re-scored on the 140-query test set)}} \\[1pt]
E1 Zero-Shot LLM   & 67.1\,{\scriptsize(59--75)} & 65.1 & --   & 0    & 91.4 & 42.9 \\
E2 RAG             & 63.6\,{\scriptsize(56--71)} & 59.6 & --   & 1    & 95.7 & 31.4 \\
E3 Single-Pass KG-agent & 57.1\,{\scriptsize(49--65)} & 48.3 & 0.0  & 1    & 15.7 & 98.6 \\
E4 KGARevion (published) & 56.4\,{\scriptsize(48--64)} & 47.8 & --   & --   & 97.1 & 15.7 \\
Meta-analysis (trim-fill) & 57.9 & 57.8 & -- & -- & 61.4 & 54.3 \\
\midrule
\multicolumn{7}{@{}l}{\emph{Fixed budget}} \\[1pt]
$k{=}3$            & 64.3\,{\scriptsize(56--72)} & 63.2 & 7.9  & 2.94  & 81.4 & 47.1 \\
$k{=}5$            & 64.3\,{\scriptsize(56--72)} & 63.5 & 11.4 & 4.72  & 78.6 & 50.0 \\
$k{=}10$           & 62.9\,{\scriptsize(55--71)} & 61.3 & 14.3 & 8.10  & 82.9 & 42.9 \\
$k{=}20$           & 61.4\,{\scriptsize(53--69)} & 59.6 & 15.7 & 11.41 & 82.9 & 40.0 \\
\midrule
\multicolumn{7}{@{}l}{\emph{Belief-convergence stopping}} \\[1pt]
KL-only            & 69.3\,{\scriptsize(61--77)} & 69.0 & 7.1  & 4.29  & 78.6 & 60.0 \\
\rowcolor[gray]{0.93}
DACG (ours)        & \textbf{69.3}\,{\scriptsize(61--77)} & \textbf{69.1} & \textbf{6.4} & 3.71 & 77.1 & \textbf{61.4} \\
\midrule
Oracle             & 77.1\,{\scriptsize(70--84)} & 76.8 & 0.0  & 3.66  & 88.6 & 65.7 \\
\bottomrule
\end{tabular}
\end{table}

\begin{table}[t]
\centering
\caption{Ablation of stopping components, test-only $n=140$. $_{\text{KL}}$: KL only; $_{\text{PRM}}$: PRM decline only.}
\label{tab:ablation}
\small
\setlength{\tabcolsep}{3.5pt}
\begin{tabular}{@{}l ccccc@{}}
\toprule
Variant & Acc\,(\%) & mF1 & Drift\,(\%) & Steps & NoE\,(\%) \\
\midrule
$k{=}20$ (no stop)  & 61.4\,{\scriptsize(53--69)} & 59.6 & 15.7 & 11.41 & 40.0 \\
DACG$_{\text{KL}}$       & 69.3\,{\scriptsize(61--77)} & 69.0 & 7.1  & 4.29  & 60.0 \\
DACG$_{\text{PRM}}$     & 65.0\,{\scriptsize(57--73)} & 64.3 & 10.7 & 6.59  & 51.4 \\
\rowcolor[gray]{0.93}
\textbf{DACG (KL+PRM)}    & \textbf{69.3}\,{\scriptsize(61--77)} & \textbf{69.1} & \textbf{6.4} & \textbf{3.71} & \textbf{61.4} \\
\midrule
Oracle                   & 77.1\,{\scriptsize(70--84)} & 76.8 & 0.0 & 3.66 & 65.7 \\
\bottomrule
\end{tabular}
\end{table}

E1 (zero-shot LLM) reaches 67.1\% overall accuracy on the test set, driven by 91.4\% on Beneficial but only 42.9\% on NoEffect. E2, which adds retrieval, drops NoEffect further to 31.4\%, illustrating that retrieval injects publication-biased evidence that reinforces rather than corrects the LLM's positive-result skew. The single-pass KG-agent (E3), the external KG paradigm run on the identical stack, presents the inverse failure: with only one retrieval step it reaches 98.6\% NoEffect but collapses to 15.7\% Beneficial, because a graph built from a single batch rarely accumulates enough concordant evidence to support a positive conclusion. The published KGARevion system (E4), run end-to-end and unmodified, attains 97.1\% Beneficial but only 15.7\% NoEffect (56.4\% overall, macro-F1 47.8), misclassifying 59 of 70 genuine nulls as Beneficial. This follows from its design: KGARevion grounds each query against a static curated KG whose relation schema (indication, contraindication, association) has no representation for the \emph{absence} of an effect, so a confirmed graph link reads as benefit. A real, peer-reviewed medical KG-agent thus systematically converts true nulls into recommendations---precisely the failure our live-evidence, drift-aware stopping is designed to prevent. E1--E4 bracket the space: zero-shot, RAG, and the static-KG agent over-predict Beneficial, the single-pass KG over-predicts NoEffect, and none balances the two classes; only adaptive stopping reaches 61.4\% NoEffect without sacrificing Beneficial (77.1\%). E1's overall accuracy (67.1\%) trails DACG's (69.3\%) only narrowly because overall accuracy rewards the majority-Beneficial skew; the decision-relevant separation is on NoEffect (+18.5\,pp) and drift, where a support tool must not convert a genuine null into a recommendation. A meta-analytic trim-and-fill correction on the retrieved polarities (57.9\% accuracy, 54.3\% NoEffect) improves on naive vote-counting but is applicable to only 75\% of queries and trails DACG substantially.

Among fixed-budget methods, increasing~$k$ from~3 to~20 does not improve overall accuracy (64.3\%$\to$61.4\%) yet drift doubles (7.9\%$\to$15.7\%) and NoEffect accuracy drops from 47.1\% to 40.0\%, consistent with Theorem~\ref{thm:drift}: a larger budget primarily increases exposure to biased evidence, and the damage falls on the null-effect class. DACG instead achieves 69.3\% overall accuracy (95\% CI 61--77, macro-F1 69.1) and 61.4\% NoEffect at 3.71 steps, with lower drift (6.4\%) than every fixed budget, improving on $k\!=\!20$ by 7.9\,pp overall and 21.4\,pp on NoEffect while using 67\% fewer steps (Fig.~\ref{fig:main}, right: belief-convergence stoppers reach the early NoEffect peak that fixed budgets overshoot).

The ablation (Table~\ref{tab:ablation}) isolates each component. KL convergence alone raises NoEffect from 40.0\% ($k\!=\!20$) to 60.0\% and halves drift to 7.1\%, the dominant signal; adding PRM decline leaves accuracy unchanged (differing on a single query) but cuts mean retrieval 4.29$\to$3.71 steps (14\%) and drift to 6.4\%. DACG$_{\text{PRM}}$ alone (51.4\% NoEffect) underperforms KL alone, confirming KL is the accuracy layer and the PRM contributes efficiency.

\begin{figure}[t]
\centering
\includegraphics[width=\columnwidth]{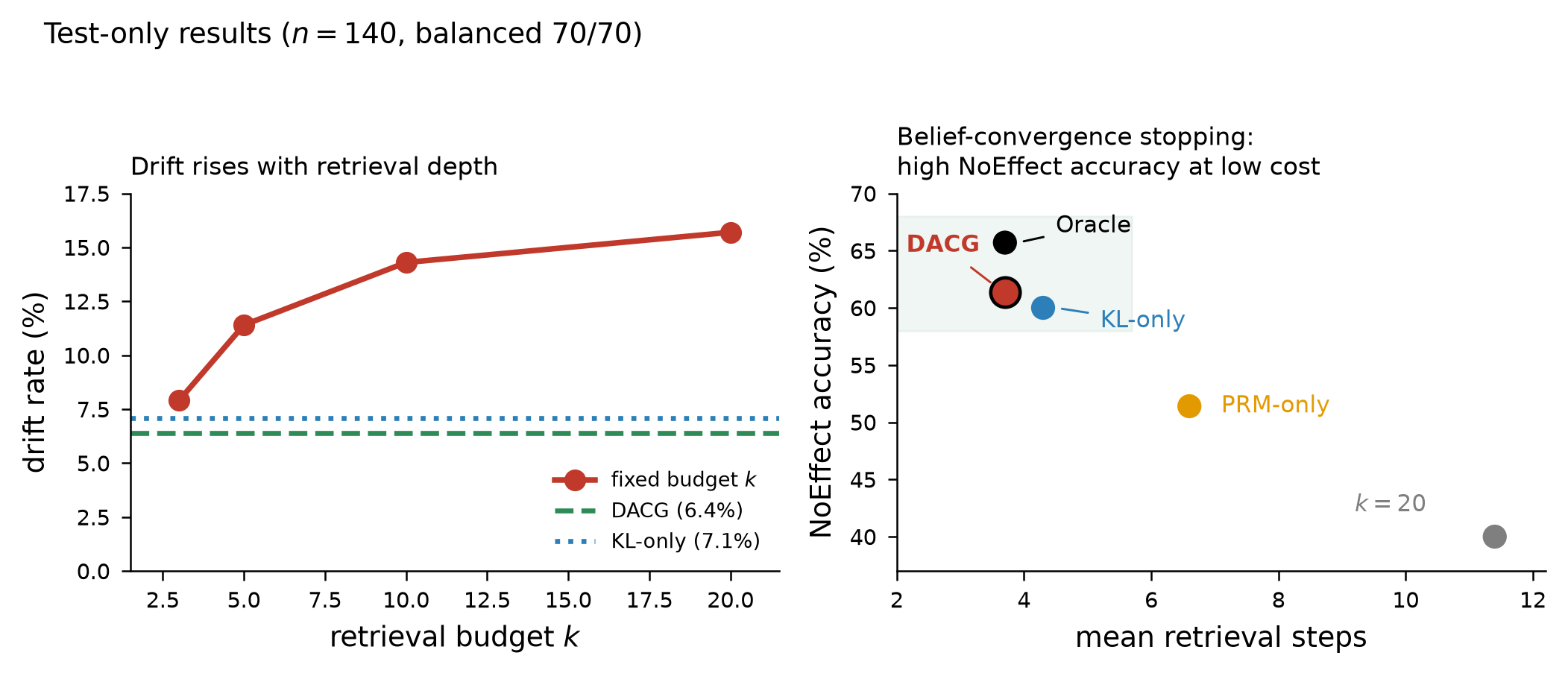}
\caption{Main results (test-only, $n=140$). Left: drift rises monotonically with retrieval budget while DACG stops early and keeps drift lowest. Right: NoEffect accuracy vs.\ mean retrieval steps; belief-convergence stopping (DACG, KL) reaches the early null-effect peak that fixed budgets overshoot.}
\label{fig:main}
\end{figure}

\subsection{Structural vs.\ Belief-Based Stopping}

Stopping on \emph{graph structure} rather than posterior stability is an intuitive alternative, but all three structural rules (density stability, path discovery, evidence plateau) confuse structural with inferential convergence and none exceeds 50.0\% NoEffect or 67.1\% overall: a structural signal cannot distinguish ``insufficient evidence, continue'' from ``true null effect, stop'', exactly the distinction belief monitoring provides (DACG: 61.4\% NoEffect, 69.3\% overall). The PRM itself separates correct from incorrect steps (mean reward $-0.082$ vs.\ $-0.368$, gap 0.287 on test; training-set margin 2.53; 88.1\% held-out pairwise accuracy), consistent with Theorem~\ref{thm:prm-bayes}.

\subsection{Statistical Tests}

Figure~\ref{fig:main} (left) plots drift against fixed budget~$k$: drift rises monotonically from 7.9\% at $k\!=\!3$ to 15.7\% at $k\!=\!20$, the empirical counterpart of Theorem~\ref{thm:drift}, and DACG (6.4\%) undercuts every fixed budget, so retrieval depth is better read as a risk variable than a monotone performance parameter. McNemar's exact test confirms DACG and KL-only each significantly outperform $k\!=\!20$ ($p = 0.019$ and $0.013$). DACG and KL-only reach the same accuracy (differing on a single query, $p = 1.0$); the PRM's role is to reach it in 14\% fewer retrieval steps.

\subsection{Cross-LLM Drift Generalisation}
\label{sec:drift-generalize}

To verify that evidence drift is not an artefact of a particular extraction model, we rebuilt two pipelines on an independent LLM (glm-4-flash): a single-pass KG-agent and a retrieval-with-reranking pipeline using a BGE cross-encoder~\cite{xiao2024cpack}. The single-pass agent reaches only 17.1\% NoEffect accuracy, and reranking only 40.0\% (the level of fixed $k{=}20$); neither approaches the belief-convergence stoppers (60--61\%). Drift thus arises from the literature's positive-result skew, not retrieval noise or extraction-model quality---better ranking alone cannot compensate for it.

\section{Related Work}
\label{sec:related}

Several biomedical KG agents have been proposed: KG-Agent~\cite{jiang2024kg} navigates pre-existing graphs via LLM tool calls; KGARevion~\cite{su2024kgarevion} integrates KG retrieval with multi-step reasoning; KG4Diagnosis~\cite{zuo2024kg4diagnosis} and KERAP~\cite{xie2025kerap} add hierarchical structure for diagnosis; MedKGent~\cite{zhang2025medkgent} constructs temporally evolving KGs from PubMed; BioKGBench~\cite{lin2024biokgbench} benchmarks such systems. None addresses distributional shift during iterative retrieval or stopping under publication bias. Publication bias itself is well documented~\cite{begg1988publication,hasenboehler2007bias,song2010dissemination}; traditional meta-analysis corrects it on a \emph{fixed} corpus, but these corrections do not transfer to incremental accumulation, the sequential case Theorem~\ref{thm:drift} characterises. Active retrieval methods~\cite{jin2023genegpt} decide \emph{whether} to retrieve but not \emph{when to stop}; the stopping problem here differs from bandit or active-learning settings because the evidence source is systematically biased, so additional samples can decrease accuracy. The energy formulation (Section~\ref{sec:approach}) draws motivation from active inference and path-based causal KG reasoning, but the formal results hold independently of such analogies.

\section{Discussion and Limitations}
\label{sec:discussion}

That the non-retrieval LLM attains competitive overall accuracy reflects publication bias absorbed during pre-training, not evidence-based reasoning: on a class-balanced test set (70/70), overall accuracy rewards methods that over-predict Beneficial. In evidence synthesis the pertinent criterion is whether conclusions rest on traceable, updateable evidence, which no zero-shot method provides; among methods that produce evidence chains, DACG achieves the best accuracy--efficiency balance.

A gap exists between the formal model and the deployed system: Theorem~\ref{thm:drift} analyses vote-counting, whereas DACG uses noisy-OR confidence and Boltzmann path scoring. The theorem isolates the minimal drift mechanism; noisy-OR adds nonlinearities (a single high-confidence positive paper pushes edge confidence near~1, amplifying drift under agreement but dampening it under conflicting-polarity entropy). The simulation of Section~\ref{sec:noisyor} and the empirical results (Section~\ref{sec:drift-empirical}) show the qualitative prediction survives these nonlinearities; a closed-form treatment under noisy-OR remains open.

Theorem~\ref{thm:prm-bayes} establishes that the population-optimal PRM identifies the trajectory argmax of correctness. In practice we use online decline detection ($r < r_{\max} - \alpha$) rather than retrospective argmax, since the full trajectory is unavailable at decision time: the theorem justifies the PRM as a ranking signal (rewards correlate with correctness: test gap 0.287, training margin 2.53) but the online rule is a conservative approximation. This partly explains why DACG$_{\text{PRM}}$ alone (51.4\% NoEffect) underperforms KL alone (60.0\%), which monitors the posterior directly.

The PRM decline rule can stop too early on Beneficial queries where evidence accumulation is genuinely helpful, contributing to DACG's slightly lower Beneficial accuracy (77.1\%) compared with KL-only (78.6\%).  Class-aware PRM thresholds or asymmetric stopping rules are a natural direction for future work.

\paragraph{Forgetting and recency.} Our claim is not that recent evidence is more biased, but that \emph{additional} evidence of any vintage from a positive-skewed corpus drives a bias-agnostic aggregator towards \textit{Beneficial} (Theorem~\ref{thm:drift}); stopping addresses the \emph{volume} of accumulation, not document age. Within a query, labels are not frozen---$q_t(L\mid Q)$ is recomputed each step and the noisy-OR update lets a later contradicting edge lower a dominant polarity---but we do not model \emph{across-time} correction (down-weighting findings a newer trial overturns). Since timestamps already enter $\phi_t$, a temporally-weighted aggregator is a concrete extension.

The entity resolver uses surface-form matching without biomedical ontologies (MeSH, UMLS~\cite{unni2022biolink}), limiting recall for synonym-rich entities; Section~\ref{sec:grounding} quantifies this (43.0\% ground to MeSH). Multiplicative polarity composition assumes hop independence~\cite{jaimini2024causallp}; mechanism-aware composition could improve path-level accuracy. The 77.1\% oracle ceiling reflects a structural limit of abstract-level synthesis, and the Harmful class is excluded (oracle accuracy 28\%) because harm mechanisms are rarely described in abstracts.

Our evaluation draws on a single benchmark family (Cochrane-derived queries). While we mitigate the sample-size concern with a larger disjoint held-out union ($n=280$; extended version) and report macro-F1 with bootstrap intervals throughout, generalisation to other evidence sources (drug-label databases, trial registries such as ClinicalTrials.gov, or non-Cochrane systematic reviews) remains to be shown. Drift is a property of the underlying publication distribution rather than of Cochrane specifically, so we expect it to appear wherever a positive-result skew exists; confirming this on a second, independently curated benchmark is the most valuable external validation and a priority for future work.

\section{Conclusion}
\label{sec:conclusion}

Publication bias creates a structural trap for retrieval-based biomedical evidence synthesis: the more literature a system retrieves, the more likely it is to misclassify a true null effect as beneficial. We formalised this as evidence drift and proved that the misclassification probability converges to one as retrieval depth grows (Theorem~\ref{thm:drift}). DACG-agent's two-layer stopping policy addresses this through two complementary layers: KL divergence monitoring bounds energy change and detects convergence, governing accuracy (Proposition~\ref{prop:kl-energy}), whilst online PRM decline detection governs efficiency by halting once evidence quality peaks. On the held-out test set, KL convergence alone reduces drift from 15.7\% to 7.1\% and raises NoEffect accuracy by 20\,pp; adding the PRM layer holds that accuracy while cutting drift to 6.4\% and using 67\% fewer retrieval steps than the full budget. In biomedical evidence synthesis under publication bias, determining when to cease retrieval is as consequential as determining what to retrieve.


{\footnotesize
\bibliographystyle{IEEEtran}
\bibliography{ref}}

\appendices
\section{Implementation Details}
\label{app:impl}
The PRM is a four-layer MLP $[20 \to 128 \to 64 \to 32 \to 1]$ (ReLU) over the 20-dimensional feature vector $\phi_t$ (six groups: path structure~6, conflict~2, coverage~4, saturation~2, global statistics~4, energy~2), trained with a Bradley--Terry pairwise loss (margin $m=0.1$) on 2{,}998 training-split preference pairs, reaching 88.1\% pairwise-ranking accuracy. Stopping thresholds, frozen after validation tuning, are KL $\delta=0.01$, PRM decline $\alpha=0.3$, and convergence $\alpha_c=0.1$; full feature definitions, the Bradley--Terry derivation, and the benchmark protocol are in the extended version.

\end{document}